\documentclass[a4paper,fleqn]{cas-dc}

\usepackage[numbers]{natbib}
\usepackage{graphicx}
\graphicspath{{figs/}}
\usepackage{picinpar}
\usepackage{amsmath}
\usepackage{amsthm}
\usepackage{url}
\usepackage{flushend}
\usepackage[utf8]{inputenc}
\usepackage{colortbl}
\usepackage{soul}
\usepackage{multirow}
\usepackage{pifont}
\usepackage{color}
\usepackage{alltt}
\usepackage{enumerate}
\usepackage{siunitx}
\usepackage{breakurl}
\usepackage{epstopdf}
\usepackage{pbox}

\usepackage{amssymb}  
\usepackage[ruled]{algorithm2e}
\usepackage{mathrsfs}
\usepackage{subcaption}
\usepackage{pgfplots}
\usepackage{booktabs}
\usepackage{xcolor}

\usepackage{hyperref}

\newcommand{\mR}{\mathbb{R}}

\newcommand{\mE}{\mathbb{E}}
\newcommand{\mP}{\mathbb{P}}

\DeclareMathOperator*{\trace}{tr}
\DeclareMathOperator*{\SE}{SE}
\DeclareMathOperator*{\SO}{SO}
\DeclareMathOperator*{\Trans}{Trans}
\DeclareMathOperator*{\Rot}{Rot}
\DeclareMathOperator*{\mat}{mat}
\DeclareMathOperator*{\support}{supp}
\DeclareMathOperator*{\mK}{\mathbb{K}}
\DeclareMathOperator*{\pconverge}{\overset{\textit{p}}{\rightarrow}}
\DeclareMathOperator*{\identitymap}{Id}
\newtheorem{theorem}{Theorem}
\newtheorem{lemma}{Lemma}

\usepackage[normalem]{ulem}

\def\tsc#1{\csdef{#1}{\textsc{\lowercase{#1}}\xspace}}
\tsc{WGM}
\tsc{QE}

\begin{document}
\let\WriteBookmarks\relax
\def\floatpagepagefraction{1}
\def\textpagefraction{.001}

\shorttitle{}

\shortauthors{}

\title [mode = title]{Registering the 4D Millimeter Wave Radar Point Clouds Via Generalized Method of Moments}

\tnotemark[1]

\tnotetext[1]{This work is supported by the National Key R\&D Program of China (Grant 2020YFC2007500), the National Natural Science Foundation of China (Grant U1813206), the Science and Technology Commission of Shanghai Municipality, China (Grant 20DZ2220400) and Natural Science Foundation of Shanghai (Grant 25ZR1401208).}

%

\author[1,2]{Xingyi Li}

\author[1,2]{Han Zhang}[orcid = 0000-0002-3905-0633]
\cormark[1]
\ead{zhanghan_tc@sjtu.edu.cn}
\cortext[1]{Corresponding author}

\author[1,2]{Ziliang Wang}
\author[3]{Yukai Yang}
\author[1,2]{Weidong Chen}

\affiliation[1]{
  organization={School of Automation and Intelligent Sensing, Shanghai Jiao Tong University},
  city={Shanghai},
  country={China}
}

\affiliation[2]{
  organization={Institute of Medical Robotics, Shanghai Jiao Tong University},
  city={Shanghai},
  country={China}
}

\affiliation[3]{
  organization={Department of Statistics, Uppsala University},
  country={Sweden}
}


\begin{abstract}
4D millimeter wave radars (4D radars) are new emerging sensors that provide point clouds of objects with both position and radial velocity measurements. Compared to LiDARs, they are more affordable and reliable sensors for robots' perception under extreme weather conditions.
On the other hand, point cloud registration is an essential perception module that provides robot's pose feedback information in applications such as Simultaneous Localization and Mapping (SLAM).
Nevertheless, the 4D radar point clouds are sparse and noisy compared to those of LiDAR, and hence we shall confront great challenges in registering the radar point clouds.
To address this issue, we propose a point cloud registration framework for 4D radars based on Generalized Method of Moments. 
The method does not require explicit point-to-point correspondences between the source and target point clouds, which is difficult to compute for sparse 4D radar point clouds.
Moreover, we show the consistency of the proposed method.
Experiments on both synthetic and real-world datasets show that our approach achieves higher accuracy and robustness than benchmarks, and the accuracy is even comparable to LiDAR-based frameworks.
\end{abstract}





\maketitle

\definecolor{limegreen}{rgb}{0.2, 0.8, 0.2}
\definecolor{forestgreen}{rgb}{0.13, 0.55, 0.13}
\definecolor{greenhtml}{rgb}{0.0, 0.5, 0.0}

\section{Introduction}
4D millimeter-wave radars (4D radars) are emerging sensors that are more affordable and weather-resilient compared to LiDARs. Hence it is attracting an increasing amount of attention in the fields of autonomous driving and robot perception. In particular, they scan the environment and provide point clouds with range, azimuth, elevation and Doppler velocity measurements. 
On the other hand, registration is a fundamental task in point clouds processing, and it has been widely exploited in robotic applications such as Simultaneous Localization and Mapping (SLAM) \cite{li20234d, zhang20234dradarslam, kim2024radar4motion, seok2025radar4voxmap}. More specifically, registration aligns multiple point cloud collections into a unified coordinate frame by inferring the relative transformation poses \cite{huang2021comprehensive}.
However, 4D radar point clouds are typically sparse and noisy compared to those of LiDARs, as illustrated in Fig. \ref{fig: radar-to-lidar}, thus making the registration challenging.

\begin{figure}
  \centering
  \includegraphics[width=0.5\textwidth]{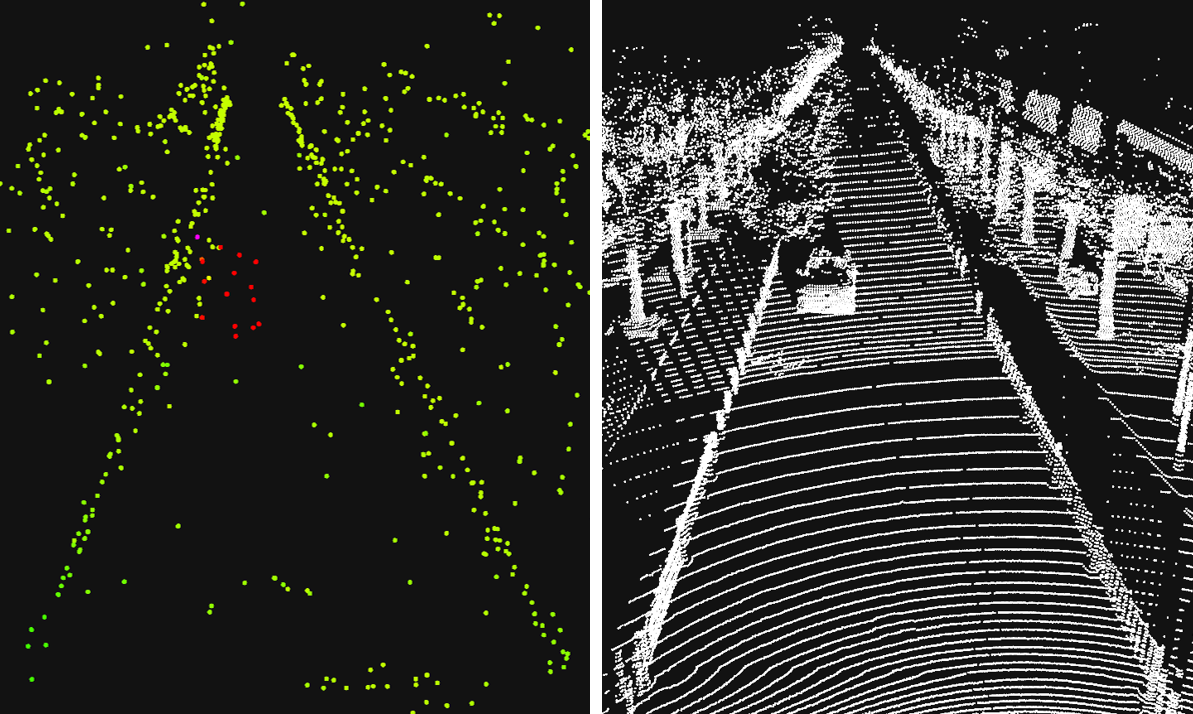}
  \caption{Comparison of 4D radar and LiDAR point clouds. Left: 4D radar point cloud colored based on the Doppler velocity. Right: LiDAR point cloud.}
  \label{fig: radar-to-lidar}
\end{figure}

Notably, the point cloud registration problems have been studied extensively from the angle of optimization. Among the widely used methods for modern SLAM systems, the most classical one is the Iterative Closest Point (ICP) algorithm \cite{besl1992method}. In particular, it iteratively establishes the nearest-neighbor correspondences between the source and target point clouds and estimates a rigid transformation that minimizes the pairwise distances. Although ICP is accurate and effective when the point clouds are clean and roughly aligned, it degrades significantly when data are noisy, sparse, or when the initial alignment is poor. To improve the registration's robustness, Generalized ICP (GICP) \cite{segal2009generalized} models each point and its local neighborhood with a covariance matrix, and it consequently combines the point-to-point and point-to-plane constraints effectively. Nevertheless, GICP still relies on iteratively built correspondences and remains sensitive to severe noise and outliers, which is exactly the case for 4D radar measurements. {More recently, Doppler ICP (DICP) \cite{hexsel2022dicp} incorporates Doppler velocity measurements directly into the ICP framework to improve registration accuracy for LiDAR data. However, it still relies on the correspondence-based ICP formulation, which may be unreliable for sparse 4D radar point clouds.} Moreover, in recent years, several methods have been developed to improve the registration's robustness under high noise and outliers. A notable method is TEASER++ \cite{yang2020teaser}, a fast yet performance-guaranteed registration framework. It decouples the rotation and the translation estimation and further employs truncated least squares for outlier rejection. However, TEASER++ relies on reliable correspondences typically obtained from feature descriptors and its performance depends on the feature quality, which is difficult to ensure for sparse and noisy 4D radar point clouds.

To avoid explicit correspondences, several distribution-level registration methods have been proposed. A well-known traditional method is the Normal Distributions Transform (NDT) \cite{magnusson2009three}, which divides the target point cloud by voxel grids. It further approximates the target distribution with a Gaussian mixture model whose components are approximated by the points within each voxel grid. Then, the source point cloud are aligned by maximizing its likelihood under the approximated Gaussian mixture model. Indeed, NDT is generally more robust than ICP and performs well on dense LiDAR data. However, its performance depends strongly on voxel size and it requires sufficiently dense local neighborhoods to estimate reliable Gaussian components. Consequently, it becomes unreliable when the point cloud is sparse, unevenly distributed, or contains large empty regions, which is the case for 4D radar data.
{Coherent Point Drift (CPD) \cite{myronenko2010point} models one point set as Gaussian mixture centroids and estimates soft probabilistic correspondences through EM optimization. GMMReg \cite{jian2010robust} represents both point sets as Gaussian mixture models and minimizes the L2 distance between the induced continuous distributions. Although these methods avoid hard nearest-neighbor correspondences, they still rely on either probabilistic correspondence estimation or continuous density overlap modeling, and their performance may degrade under the extreme sparsity and noise of 4D radar data.
Our formulation is also conceptually connected to the kernel mean embedding and Maximum Mean Discrepancy (MMD) \cite{smola2007hilbert, gretton2012kernel, muandet2017kernel}, which represent and compare probability distributions through their expectations of kernel functions in a reproducing kernel Hilbert space (RKHS). Unlike MMD, which measures distribution discrepancy through the full RKHS norm involving pairwise kernel evaluations among all points, our method matches the kernel mean embeddings at a finite set of selected centers within the generalized method of moments framework and explicitly incorporates the rigid transformation into the moment matching process. Moreover, our work provides a statistical consistency analysis under the registration setting, and integrates the method into practical radar SLAM systems.}

On the other hand, feature-based methods extract keypoints and descriptors from the point clouds to match corresponding parts of the two shapes. These methods are less sensitive to noise and partial overlaps because they do not require the exact fit of every point in the point clouds. One commonly used local descriptor is Fast Point Feature Histogram (FPFH) \cite{rusu2009fast}, which captures the geometric relationships between a point and its neighbors. Such local descriptor is fast and hence is widely used for coarse alignment. However, FPFH is sensitive to noise and not very discriminative in flat areas. Another popular descriptor is Signature of Histograms of Orientations (SHOT) \cite{salti2014shot}, which is more robust to rotation and captures more detailed local geometry. SHOT performs better than FPFH in complex scenes, but it is also more computationally expensive. These feature-based methods work well when the points have rich geometric structure and sufficient resolution. However, the 4D radar point cloud measurements are very sparse, noisy and lacks strong features, these descriptors will fail to produce reliable feature matches.

In addition, deep learning has recently been applied to point cloud registration to improve its robustness and exempt the need for hand-crafted features. In particular, these methods use neural networks to learn the features and the matching strategies directly from data. For example, Deep Closest Point (DCP) \cite{wang2019deep} uses feature embeddings and a transformer-based attention module to estimate the soft correspondences. It achieves a strong performance on dense point clouds. However, because DCP relies on the learned dense and reliable features, its generalization degrades when the point clouds are sparse and noisy, which makes it less suitable for 4D radar data. Moreover, PointNetLK \cite{aoki2019pointnetlk} combines the PointNet \cite{qi2017pointnet} architecture with a classical Lucas-Kanade optimization process. It learns the global features and updates the transformation iteratively without requiring explicit correspondences. However, the use of a numerical Jacobian would lead to instability and poor generalization. To address these limitations, PointNetLK Revisited \cite{li2021pointnetlk} replaces the numerical Jacobian with an analytical decomposition, thus improving the stability and allowing better feature reuse. Nevertheless, it still relies on supervised training on large datasets and its performance degrades when facing unseen environments. As a summary, the performance of current learning-based approaches may degrade considerably when applied to sparse and highly noisy 4D radar point clouds.

To address this issue, recent 4D radar SLAM systems have introduced radar-specific features such as measurement uncertainty, Radar Cross Section (RCS) information, and Doppler velocity to improve the registration performance \cite{li20234d, zhang20234dradarslam, kim2024radar4motion, seok2025radar4voxmap}. For instance, 4DRaSLAM employs scan-to-submap NDT augmented with radar point measurement uncertainty \cite{li20234d}, 4DRadarSLAM enhances GICP through adaptive covariance modeling \cite{zhang20234dradarslam}, Radar4Motion applies RCS-weighted ICP and accumulated scan matching \cite{kim2024radar4motion}, and Radar4VoxMap incorporates Doppler residuals into a GICP formulation \cite{seok2025radar4voxmap}. These approaches improve robustness in practice, yet they still rely on classical GICP's correspondence-based optimization or NDT’s voxelized Gaussian mixture modeling. Given the sparse and noisy characteristics of the 4D radar point clouds, achieving a robust and precise registration for existing methods still remains challenging. Hence we are motivated to design 
a robust, correspondence-free, distribution-level registration module for the sparse, noisy and unstable 4D radar point clouds and facilitate the practical use of 4D radars.
In particular, we propose a Generalized Method of Moments Registration (GMMR) framework in this work.
Our contribution is three-fold:
\begin{itemize}
     \item \textbf{A correspondence-free registration framework based on the Generalized Method of Moments.} The proposed registration framework leverages Generalized Method of Moments (GMM) \cite{hansen1982large} to align the points distributions observed in the source and target frames and estimate the transformation between the frames. It does not rely on any explicit correspondences between the points. To the best of our knowledge, this is the first work that rigorously applies GMM to point cloud registration and further in 4D radar SLAM.
    \item \textbf{The theoretical consistency guarantees.} We show the proposed registration algorithm is statistically consistent, that is, the transformation estimation converges in probability to the ``true" transformation as the number of points increases. 
    This serves as the robust guarantee of the algorithm against noise when the numbers of points in clouds are relatively sufficient.
    \item \textbf{The comprehensive evaluation on both synthetic and real-world datasets.} We conduct extensive experiments on both synthetic datasets and real-world 4D radar sequences to test our method's performance. The results show that our method consistently outperforms commonly used registration baselines and achieves an accuracy that is comparable to LiDAR-based frameworks despite the significantly lower resolution and heavier noise in radar data.
\end{itemize}

\section{Problem formulation}

Consider a static environment containing multiple objects. A 4D radar captures this scene from two different viewpoints, referred to as the source frame and the target frame. 
Consequently, two point clouds are obtained by sampling from the same surface represented in different coordinate systems.

To formulate this, let $(\mR^3,\mathcal{B}(\mR^3),\mu)$ be the probability space carried by the target frame, where $\mathcal{B}(\mR^3)$ is the Borel $\sigma$-algebra in $\mR^3$, and $\mu$ is the probability measure such that $\mu(A)$ is the probability that the point cloud lies in $A\subset \mR^3$ in the target frame. And the support of $\mu$ is $\support(\mu):=\mK$.
Moreover, let $(\mR^3,\mathcal{B}(\mR^3))$ be the measurable space in the source frame. Furthermore, let $T: \mR^3\times \SE(3) \rightarrow \mR^3$ be the rigid transformation, i.e., for $\mathbf{x}\in\mR^3$ in the source frame and $\mathbf{y}\in\mR^3$ in the target frame,
\begin{align*}
    \mathbf{y} = T(\mathbf{x};\underbrace{\left(\mathbf{R},\mathbf{t}\right)}_{\theta})=\mathbf{R}\mathbf{x}+\mathbf{t},
\end{align*} 
where $(\mathbf{R}, \mathbf{t}) \in \SE(3)$, $\mathbf{R} \in \SO(3)$ is a rotation matrix and $\mathbf{t} \in \mathbb{R}^3$ is a translation vector. The ``true" rigid transformation that transforms the source to the target frame is denoted as $T_0:=T(\;\cdot\; ;\theta_0)$, where $\theta_0:=(\mathbf{R}_0,\mathbf{t}_0)$ is the ``true" rotation matrix and translation vector. Clearly, $T(\cdot;\theta)$ is a measurable function for all $\theta\in\SE(3)$ and we denote $\mu\circ T_0$ as the probability measure induced by $T(\cdot;\theta_0)$. In particular, for any $B\subseteq\mR^3$, it holds that
$(\mu\circ T_0)(B) := \mu(T(B;\theta_0))$, where $T(B;\theta_0)$ is the image of $B$ under $T_0$.
Consequently, we model the point clouds $\mathcal{X}$ and $\mathcal{Y}$ observed in the source and target frames, respectively, as
\begin{itemize}
    \item $\mathcal{Y} = \{ y_j\in\mR^3 \}_{j=1}^M $, where $y_j$ is the realization of the random vector $\mathbf{y}_j\overset{\text{i.i.d.}}{\sim} \mu $;
    \item $ \mathcal{X} = \{ x_i \in\mR^3\}_{i=1}^N $, where $x_i$ is the realization of the random vector $\mathbf{x}_i\overset{\text{i.i.d.}}{\sim} \mu\circ T_0 $.
\end{itemize}
Equipped with the above formulation, we aim to estimate the ``true" parameter $\theta_0$ that defines the ``true" rigid transformation $T_0$ with the point clouds $\mathcal{X}$ and $\mathcal{Y}$. 

\section{GMM-based point cloud registration}
In this section, we introduce our point cloud registration method using GMM. First, we explain how to compute the generalized moments with Gaussian RBF to represent the distribution of a point cloud. Next, we present the optimization process to estimate the transformation which aligns point clouds by minimizing the difference between their moments. Finally, we introduce some implement details, including a CUDA-based parallelization strategy for efficient computation and an ego-motion–aided overlap extraction method to improve distribution consistency.

\subsection{The generalized moments and its approximations}
To register the point clouds and estimate $\theta_0$, we need to compare and match the ``features" of the distributions $\mu$ and $\mu\circ T_0$.
To this end, we choose to match the generalized moments of $\mu$ and $\mu\circ T_0$. In general, a moment can be defined as the expectation of a measurable function with respect to a probability or finite measure. 
More specifically, let $\nu$ be a nonnegative Borel measure defined on $\support(\nu):=\mK\subseteq\mathbb{R}^d$, and let $\phi_k: \mathbb{R}^d \rightarrow \mathbb{R}$ be a measurable kernel function. The $k$-th generalized moment with respect to $\nu$ is defined as
\begin{equation}
    m_k^\nu = \int_{\mK} \phi_k(\mathbf{x}) \, \nu(\mathrm{d}\mathbf{x}). \nonumber
\end{equation}
This formulation includes classical monomial moments (i.e., $\phi_k(\mathbf{x}) = x^{k_x}y^{k_y}z^{k_z}$, where $\mathbf{x}=[x,y,z]^T$ and $k_x+k_y+k_z=k$) as well as more general moments constructed from non-monomial functions. When $\nu$ is a probability measure, this reduces to expectation $m_k^\nu = \mathbb{E}_{\nu}[\phi_k(\mathbf{x})]$.

In the case of point cloud registration, we do not know the explicit distribution $\mu$. Hence, it is difficult to calculate the generalized moments explicitly. Nevertheless, since the point clouds $\mathcal{X}$ and $\mathcal{Y
}$ are i.i.d. sampled from the probability measures $\mu\circ T$ and $\mu$, we approximate the expectations using empirical averages. 
More specifically, the empirical $k$-th moment of the target point cloud $\mathcal{Y}$ takes the form
\begin{align}
    m_k^\mu=\mE_\mu[\phi_k(\mathbf{y})]\approx 
    \frac{1}{M} \sum_{j=1}^{M} \phi_k(\mathbf{y}_j):=\hat m_k(\mathcal{Y}). \label{eq: target moment}
\end{align}
Moreover, by applying the change of variable $\mathbf{y}=T(\mathbf{x};\theta_0)$, it also holds for the generalized moment $m_k^\mu$ in \eqref{eq: target moment} that 
\begin{align}
    & m_k^\mu = \int_{\mK} \phi_k(\mathbf{y})\mu(\mathrm{d}\mathbf{y})=\int_{\mK} \phi_k(T(\mathbf{x};\theta_0))\mu(\mathrm{d}T(\mathbf{x};\theta_0))\nonumber\\
    = & \int_{T_0^{-1}(\mK)} \phi_k\!\left(T\left(\mathbf{x};\theta_0\right)\right)(\mu\circ T_0)(\mathrm{d}\mathbf{x})=\mE_{\mu\circ T_0}[\phi_k(T(\mathbf{x};\theta_0))]\nonumber\\
    \approx & \frac{1}{N}\sum_{i=1}^N \phi_k\left(T\left(\mathbf{x}_i;\theta_0\right)\right)
    :=\hat{m}_k\left(T\left(\mathcal{X};\theta_0\right)\right),\label{eq: trans source moment}
\end{align}
where the approximation lies in the fact that $\{\mathbf{x}_i\}_{i=1}^N$ are i.i.d. samples from the measure $\mu\circ T_0$ and we refer $\hat{m}_k\left(T\left(\mathcal{X};\theta_0\right)\right)$ as empirical generalized moments.
As we shall see in Section \ref{sec: trans est}, we will minimize the difference between the empirical moments with respect to $\theta$ to estimate the ``true" rigid transformation $\theta_0$. 

\subsection{Kernel function design}
The choice of kernel function $\phi_k$ plays an important role and it affects the sensitivity and robustness of the final registration result. In this work, we adopt the Gaussian Radial Basis Function (RBF) as our kernel function, given by,
\begin{equation}
    \phi_k(\mathbf{x}) = e^{-(\mathbf{x} - \mathbf{c}_k)^\top \Sigma^{-1} (\mathbf{x} - \mathbf{c}_k)}, \label{eq: gaussian rbf}
\end{equation}
where $\mathbf{c}_k \in \mathbb{R}^3$ is the center of the kernel, and $\Sigma$ determines the width of the ``Gaussian bell". 
Intuitively speaking, each generalized moment with the kernel \eqref{eq: gaussian rbf} captures the local features of the distributions $\mu$ around the center $\mathbf{c}_k$, since the value of $\phi_k(\mathbf{x})$ decays exponentially fast as $\mathbf{x}$ leaves the center $\mathbf{c}_k$. And by placing the centers $\mathcal{P}_c:=\{\mathbf{c}_k\}_{k=1}^\kappa$ in the 3D space of point cloud, we can capture the global ``features" in the points distribution.

Next, we shall choose the centers $\{\mathbf{c}_k\}_{k=1}^\kappa$ and $\Sigma$ to ensure the effectiveness and efficiency of such ``feature capture". To this end, we adopt an adaptive strategy for selecting the centers of the Gaussian RBF kernels based on the density of the input point cloud. More specifically, if the point cloud is sparse or contains points fewer than a predefined threshold, then each point will serve as a center.
This ensures a more effective ``feature capture" using the generalized Gaussian RBF moments -- if the centers are not placed carefully when the point clouds are sparse, then most of the empirical Gaussian RBF moments would be almost zero. 
On the other hand, for denser point clouds, we apply $k$-means clustering to select a compact set of representative centers.
This ensures that the centers are placed at where the points are dense to extract local ``features" of the point cloud distribution, 
while keeping a reasonable number of centers (order of moments) so that we can have a tolerable computational cost.
After we choose the centers $\{\mathbf{c}_k\}_{k=1}^\kappa$, the parameter $\Sigma$ is determined empirically in our experiments.

\subsection{Transformation estimation} \label{sec: trans est}
We now start to construct the optimization problem that matches the moments and estimate the ``true" rigid transformation $\theta_0$. 
Note that as expressed in \eqref{eq: target moment} and \eqref{eq: trans source moment}, both $\mE_\mu[\phi_k(\mathbf{y})]$ and $\mE_{\mu\circ T_0}[\phi_k(T(\mathbf{x};\theta_0))]$ equal to $m_k^\mu$.
Hence, the two empirical approximations $\hat{m}_k(\mathcal{Y})$ and $\hat{m}_k(T(\mathcal{X};\theta_0))$ should coincide since they both approximate the same $m_k^\mu$. Therefore, we can estimate the optimal transformation by minimizing the moment matching loss, which is the difference between the empirical moments of the transformed source and those of the target with respect to the parameter $\theta$.
In particular, we use the least-square loss 
\begin{align}
&\mathcal{L}(\theta) = \sum_{k=1}^\kappa \|\mE_{\mu\circ T}[\phi_k(T(\mathbf{x};\theta))]-\mE_\mu[\phi_k(\mathbf{y})]\|^2\nonumber\\
&\approx\sum_{k=1}^{\kappa} \left\| \hat{m}_k(T(\mathcal{X};\theta)) - \hat{m}_k(\mathcal{Y}) \right\|^2:=\mathcal{L}_{M,N}(\theta), \label{eq: cost function}
\end{align}
where $\hat{m}_k(\cdot)$ denotes the $k$-th order empirical moment of order $k$.
By substituting the expressions \eqref{eq: target moment}, \eqref{eq: trans source moment} and \eqref{eq: gaussian rbf} into \eqref{eq: cost function}, we obtain the optimization that we actually solve to estimate the rigid transformation $\theta_0$
\begin{align}
&\min_{\theta\in\mathscr{B}(\eta)} \mathcal{L}_{M,N}(\theta)\!=\!\sum_{k=1}^{\kappa}  [ \frac{1}{N}\!\sum_{i=1}^{N} 
e^{-(\mathbf{R} \mathbf{x}_i + \mathbf{t} - \mathbf{c}_k)^\top \Sigma^{-1} (\mathbf{R} \mathbf{x}_i + \mathbf{t} - \mathbf{c}_k)}  \nonumber \\
& \qquad \qquad  - \frac{1}{M} \sum_{j=1}^{M} e^{-(\mathbf{y}_j - \mathbf{c}_k)^\top \Sigma^{-1} (\mathbf{y}_j - \mathbf{c}_k)} ]^2,\label{eq: cost function detailed}
\end{align}
where $\mathscr{B}(\eta):=\{\theta\:|\:(\mathbf{R},\mathbf{t})\in\SE(3),\|\mathbf{t}\|^2\le\eta\}$ is the compact feasible domain for $\theta$, and $\eta$ is a pre-defined constant that can be arbitrarily large.

Suppose $\hat\theta_{M,N}$ is the minimizer of \eqref{eq: cost function detailed},  we would like to state the consistency of such estimate, which we refer to the following theorem. This is the first robust guarantee of the algorithm against noise when the numbers of points in clouds are relatively sufficient.
To this end, we first present the following lemma.
\begin{lemma}\label{lem:injective}
Let $\mathbf\Phi(\mathbf{x}):=[\phi_1(\mathbf{x}),\phi_2(\mathbf{x}),\ldots,\phi_\kappa(\mathbf{x})]^\top$ be a $\kappa$-dimensional vector-valued kernel function where $\phi_k$'s are defined as \eqref{eq: gaussian rbf}. If the centers $\{\mathbf{c}_k\}_{k=1}^\kappa$ do not lie in any hyperplane in $\mR^3$, then $\mathbf\Phi$ is injective.
\end{lemma}
\begin{proof}
We use contradiction proof. Assume $\exists \mathbf{x}_1\neq \mathbf{x}_2$ such that $\mathbf\Phi(\mathbf{x}_1)=\mathbf\Phi(\mathbf{x}_2)$. This implies $\|\mathbf{x}_1-\mathbf{c}_k\|_{\Sigma^{-1}}^2=\|\mathbf{x}_2-\mathbf{c}_k\|_{\Sigma^{-1}}^2$ holds for $k=1,\ldots,\kappa$. And it follows that $2\Delta\mathbf{x}^\top\Sigma^{-1}(\mathbf{x}_2-\mathbf{c}_k)+\Delta\mathbf{x}^\top\Sigma^{-1}\Delta\mathbf{x}$, where $\Delta \mathbf{x}:=\mathbf{x}_1-\mathbf{x}_2=0$ for all $k=1,\ldots,\kappa$.
This implies $\mathbf{b}^\top\mathbf{c}_k=\mbox{constant}$ for all $k=1,\ldots,\kappa$, where $\mathbf{b}=\Sigma^{-1}\Delta\mathbf{x}$ is a nonzero constant vector. Hence $\mathbf{c}_k$ lives in a hyperplane, which contradicts the assumption.
\end{proof}
Now we are ready to present the main theorem of this work.
\begin{theorem}[Statistical consistency of the estimator]\label{thm: main_thm}
Suppose that the centers $\{\mathbf{c}_k\}_{k=1}^\kappa$ do not lie in any hyperplane in $\mR^3$. Then, the estimator $\hat{\theta}_{M,N}\pconverge\theta_0$ as $M\rightarrow\infty$ and $N\rightarrow\infty$.
\end{theorem}
\begin{proof}
    Let $ z_1 \sim \text{Bernoulli}\left( \frac{N}{N + M} \right) $ be a binary random variable indicating whether a sample is drawn from the source or target distribution. Let $ \mathbf{z}_2 \in \mathbb{R}^3 $ be a spatial point such that $ \mathbf{z}_2 \sim \mu\circ T_0 $ when $ z_1 = 0 $, and $ \mathbf{z}_2 \sim \mu$ when $ z_1 = 1 $.
Further denote $\mathbf{z}:=[z_1,\mathbf{z}_2^\top]^\top$.
Furthermore, let $\mathbf{g}:\mR^4\times \SE(3)\rightarrow \mR^\kappa$ be
\begin{equation*}
\mathbf{g}(\mathbf{z};\theta) = (1 - z_1) \cdot \Phi(\mathbf{R} \mathbf{z}_2 + \mathbf{t}) \cdot \frac{N + M}{N} - z_1 \cdot \Phi(\mathbf{z}_2) \cdot \frac{N + M}{M}.
\end{equation*}
Next, take expectation of the above function with respect to $\mathbf{z}$, we get
\begin{align*}
    &\mathbb{E}[\mathbf{g}(\mathbf{z}, \theta)] = \frac{N+M}{N}\mE[\mathbf{\Phi}(\mathbf{R} \mathbf{z}_2 + \mathbf{t})|z_1=0]\cdot\underbrace{\mP(z_1=0)}_{\frac{N}{N+M}}\\
    &+\frac{N+M}{N}\mE[\mathbf{\Phi}(\mathbf{R} \mathbf{z}_2 + \mathbf{t})|z_1=1]\cdot 0\cdot\mP(z_1=1)\\
    &-\frac{N+M}{M}\mE[\mathbf{\Phi}(\mathbf{z}_2)|z_1=0]\cdot 0\cdot\mP(z_1=0)\\
    &-\frac{N+M}{M}\mE[\mathbf{\Phi}(\mathbf{z}_2)|z_1=1]\cdot\underbrace{\mP(z_1=1)}_{\frac{M}{N+M}}\\
    &=\mE_{\mu\circ T_0}[\mathbf{\Phi}(T(\mathbf{x};\theta))]-\mE_\mu[\mathbf{\Phi}(\mathbf{y})].
\end{align*}
Consequently, the ``original" objective function $\mathcal{L}(\theta)$ in \eqref{eq: cost function} that we are minimizing can be rewritten into the following compact form
$\mathcal{L}(\theta)=\|\mE[\mathbf{g}(\mathbf{z};\theta)]\|^2$.
This falls exactly into the scope of GMM \cite{hansen1982large}, where we approximate the expectation with empirical mean. Hence we can prove the consistency by checking whether the consistency conditions for GMM are satisfied. In particular,

\begin{itemize}
    \item The parameter domain $\theta\in\mathscr{B}(\eta)$ is compact. Recall the definition of $\theta\in\mathscr{B}(\eta)$, the space $\SO(3)$ for rotation matrix $\mathbf{R}$ is clearly compact. Moreover, we force $\|\mathbf{t}\|^2\le\eta$ in $\theta\in\mathscr{B}(\eta)$ because in practice the translation is often bounded to a known range, and hence the statement follows.
    \item $\mathbf{g}(\mathbf{z}; \theta)$ is continuous at each $\theta$ with probability one. This holds from the fact that the kernel function $\mathbf\Phi$ is continuous, and the transformation $T(\cdot;\theta)$ acts continuously on $\mathbf{z}_2 \in \mathbb{R}^3$. Therefore, for every realization of $\mathbf{z}$, the mapping $T \mapsto \mathbf{g}(\mathbf{z};\theta)$ is continuous.
    \item $\mathbb{E}[\sup_{\theta \in \mathscr{B}(\eta)} \|g(\mathbf{z};\theta)\|] < \infty$. This holds from the fact that the kernel function $\Phi$ is bounded. Specifically, since $\Phi$ consists of $K$ Gaussian radial basis functions, each satisfying $0 < \phi_k(\mathbf{x}) \le 1$. Hence $\| \Phi(\mathbf{x}) \| \le \sqrt{K}$. In addition, $z_1$ is either 0 or 1 since it is Bernoulli distributed. These together give the statement.
    \item $\mE[\mathbf{g}(\mathbf{z};\theta)]=0$ has a unique solution $\theta_0$.
    Let $\mE[\mathbf{g}(\mathbf{z};\theta)]=0$ and this implies 
    \begin{equation*}
        \mE[\mathbf{g}(\mathbf{z};\theta)]=\mE_{\mu\circ T_0}[\mathbf{\Phi}(T(\mathbf{x};\theta))]-\mE_\mu[\mathbf{\Phi}(\mathbf{y})]=0.
    \end{equation*}
     On the other hand, by the definition of $\mu\circ T_0$, we can write $\mE_{\mu\circ T_0}[\Phi(T(\mathbf{x};\theta))]$ as
    \begin{align*}
        &\mE_{\mu\circ T_0}[\Phi(T(\mathbf{x};\theta))] \\
        &=\int_{T_0^{-1}(\mK)}\mathbf\Phi(T(\mathbf{x};\theta))(\mu\circ T_0)(\mathrm d\mathbf{x}) \\
        &=\int_{\mK}\mathbf\Phi(T(T_0^{-1}(\mathbf{y});\theta))\mu(\mathrm d\mathbf{y}),
    \end{align*}
    which further equals to $\mE_\mu[\mathbf\Phi(\mathbf y)]=\int_{\mK}\mathbf\Phi(\mathbf y)\mu(\mathrm{d}\mathbf y)$.
    This implies $\mathbf\Phi(T(T_0^{-1}(\:\cdot\:);\theta))=\mathbf\Phi(\cdot)$, $\mu$-a.s.. By Lemma \ref{lem:injective}, this further implies $T_\theta\circ T_0^{-1}=\identitymap$ and hence $\theta = \theta_0$.
\end{itemize}
Therefore, all of the consistency conditions for GMM are satisfied and we have $\hat{\theta}_{M,N}\pconverge \theta_0$ as $M,N\rightarrow\infty$.
\end{proof}

Our objective function $\mathcal{L}_{M,N}(\theta)$ is differentiable with respect to the parameter $\theta$. To efficiently minimize $\mathcal{L}_{M,N}(\theta)$, we employ the Broyden-Fletcher-Goldfarb-Shanno (BFGS) quasi-Newton method \cite{dennis1977quasi, head1985broyden}.
{Although $\mathcal{L}_{M,N}(\theta)$ is nonconvex, the risk of convergence to local minima could be reduced in practice, since the inter-frame motion is typically small and a good initial guess is available from ego-motion estimation or the transformation from the previous frame. Moreover, we parameterize the rotation matrix $\mathbf{R} \in \SO(3)$ using Euler angles, which converts the constrained optimization over rotation matrices into an unconstrained optimization by avoiding explicit orthogonality constraints and offering relatively simple analytical derivatives. While gimbal lock may occur near $\pm 90^\circ$ pitch, this is unlikely in typical ground robot scenarios with small inter-frame rotations and predominantly planar motion.}
In addition, the expression of gradient $\frac{\partial}{\partial \theta}\mathcal{L}_{M,N}(\theta)$ is pre-computed via chain rule. That is, the derivative of the empirical moment terms with respect to the Euler angles $\psi_i$ takes the form $\frac{\partial}{\partial \psi_i}\hat{m}_k(T(\mathcal{X};\theta)) = \frac{\partial}{\partial \mathbf{R}}\hat{m}_k(T(\mathcal{X};\theta))\frac{\partial }{\partial \psi_i}\mathbf{R}$.
{
The iterative BFGS optimization terminates when one of the following criteria is met: (i) the number of iterations exceeds a maximum number $N_{\max}$; (ii) the norm of the gradient falls below a gradient convergence threshold $\epsilon_g$, i.e., $\|\nabla_\theta \mathcal{L}_{M,N}(\theta)\| \leq \epsilon_g$; or (iii) the incremental transformation update between consecutive iterations becomes sufficiently small, i.e.,
\begin{equation}
    \|\Delta\mathbf{t}\| \leq \epsilon_t \quad \text{and} \quad \arccos \frac{\trace(\Delta\mathbf{R})-1}{2} \leq \epsilon_r, \nonumber
\end{equation}
where $\Delta\mathbf{R} \in \SO(3)$ and $\Delta\mathbf{t} \in \mathbb{R}^3$ are the rotation and translation components of the incremental transformation update, $\epsilon_t$ is the translation convergence threshold, and $\epsilon_r$ is the rotation convergence threshold.
}
In summary, the transformation estimation process is shown in Algorithm \ref{alg: gmmr}.

\SetKwFor{ParForEach}{forall}{do in parallel}{end for}
\begin{algorithm}[!htpb]
\caption{Generalized Method of Moments Registration} \label{alg: gmmr}
\KwIn{Source point cloud $\mathcal{X}$, Target point cloud $\mathcal{Y}$}
\KwOut{The estimate $\hat{\theta}_{M,N}$ of the optimal rigid transformation parameter $\theta_0$.
}

\textbf{Initialization:} \\
Allocate center points $\mathcal{P}_c:=\{\mathbf{c}_k\}_{k=1}^\kappa$ \\
\ParForEach{center $\mathbf{c}_k \in \mathcal{P}_c$}{
    \ForEach{point $\mathbf{y}_j \in \mathcal{Y}$}{
        Compute $\phi_k(\mathbf{y}_j)$
    }
    {$\hat{m}_k(\mathcal{Y}) \leftarrow \frac{1}{M} \sum_{j=1}^{M} \phi_k(\mathbf{y}_j)$}
}

\textbf{Registration:} \\
\While{not converged}{
    $ \mathcal{L}_{M,N}(\theta) \leftarrow 0 $ \\
    Initialize transformation $T(\cdot;\theta)$ \\
    \ParForEach{center $\mathbf{c}_k \in \mathcal{P}_c$}{
        \ForEach{point $\mathbf{x}_i \in \mathcal{X}$} {
            Compute $\phi_k(T(\mathbf{x}_i;\theta))$
        }
        $\hat{m}_k(T(\mathcal{X};\theta)) \leftarrow \frac{1}{N} \sum_{i=1}^{N} \phi_k(T(\mathbf{x}_i;\theta))$ \\
        $\Delta_k(\theta) \leftarrow  (\hat{m}_k(T(\mathcal{X};\theta)) - \hat{m}_k(\mathcal{Y}))^2$ \\
        Compute gradient $\nabla_\theta \Delta_k(\theta)$ \\
    }
    $ \mathcal{L}_{M,N}(\theta) \leftarrow \sum_{k=1}^{\kappa} \Delta_k(\theta)$ \\
    $ \nabla_\theta \mathcal{L}_{M,N}(\theta) \leftarrow \sum_{k=1}^{\kappa} \nabla_\theta \Delta_k(\theta)$ \\
    Update $\theta$ using BFGS method \\
}
\Return $\hat{\theta}_{M,N}$
\end{algorithm}

The proposed method is inherently robust to both noise and sparsity. 
In particular, since the effect of the potential measurement noise is intrinsically included in the point distributions, and the registration is based on the statistical properties of the distribution instead of any specific structural features like ``lines", ``planes" or ``the nearest neighbor", etc., the method is intrinsically robust to noisy point clouds.
Moreover, the choice of Gaussian RBF makes the contribution of outliers to the moments extremely small due to the exponential decay of kernel function. Hence our method also reduces the sensitivity to outliers.
For sparse points, our method avoids reliance on local statistics, which typically requires a sufficient number of points within voxel grids to obtain accurate approximations of the statistics such as mean and covariance as in NDT.
The proposed method utilizes generalized moments that is based on the entire point cloud, enabling a robust registration for sparse point clouds.

To further enhance computational efficiency, we implement the moment computations and transformation updates using CUDA parallelization, as detailed in next section.

\subsection{Practical implementation details}
\subsubsection{CUDA acceleration}

In order to improve the computational efficiency of our GMMR framework, we implement the key components of the algorithm using CUDA-based parallel computing on GPU for acceleration.

In particular, the most computationally intensive parts of our method involve
\begin{itemize}
    \item Computing the Gaussian RBF kernel values $\phi_k(\cdot)$ for all points and centers.
    \item Evaluating the moment matching loss $\mathcal{L}_{M,N}(\theta)$ and its gradient $\nabla_\theta \mathcal{L}_{M,N}$ during optimization.
\end{itemize}

By leveraging the parallel computing on GPUs, we assign each CUDA thread to compute $\phi_k(\cdot)$ for a specific center. This allows for highly parallel and efficient calculation of kernel values. Similarly, for the optimization phase, the gradient computation is also parallelized by dividing the contributions of different centers across multiple threads. This enables us to efficiently update the transformation parameters during each iteration.

Our CUDA implementation achieves a significant reduction in runtime. It improves the computational efficiency of our registration method and enables practical performance for radar point cloud registration tasks.

\subsubsection{Ego-motion-based overlap extraction}

Our registration framework assumes that the source and the target point clouds are sampled from the same underlying spatial distributions. However, different frames of the 4D radar scans may contain only partial overlap due to ego-motion and the radar’s limited Field of View (FoV). Consequently, directly applying registration on two point clouds containing large non-overlapping regions violates the assumption and may degrade the registration accuracy. To mitigate this issue, we first estimate the relative motion using Doppler-based ego-velocity estimation, inspired by \cite{kellner2013instantaneous}, and then extract the overlap region after compensating for the motion.

In particular, for each radar point in the point cloud, 4D radar provides the relative radial Doppler velocity information between the radar and the measured radar point. For a static point $\mathbf{P}_i$ in the world frame, its radial Doppler velocity $v_{r,i}$ is just the inverse of the radar's ego-velocity $\mathbf{v}_s$. Hence let $\mathbf{d}_i$ denote the unit direction vector of $\mathbf{P}_i$ in the radar frame, the radial Doppler velocity measurement of the static point $\mathbf{P}_i$ shall satisfy
\begin{equation}
v_{r,i} = -\mathbf{d}_i^{\mathsf{T}} \mathbf{v}_s.
\label{eq: doppler relation}
\end{equation}

Moreover, under a rigid mounting, radar velocity $\mathbf{v}_s$ is a function of the vehicle's linear and angular ego-velocities $(\mathbf{v}, \boldsymbol{\omega})$. More specifically,
\begin{equation}
\mathbf{v}_s = \mathbf{R}_s^{\mathsf{T}}
    \left( \mathbf{v} + 
    \boldsymbol{\omega} \times \mathbf{t}_s \right),
\label{eq: radar velocity}
\end{equation}
where $\mathbf{R}_s$ and $\mathbf{t}_s$ denote the extrinsic parameters from radar to the vehicle's ego-velocity frame. By stacking \eqref{eq: doppler relation} and \eqref{eq: radar velocity} for all static points, rejecting dynamic outliers with RANSAC, and solving a linear least-squares problem, we can obtain an estimate $(\widetilde{\mathbf{v}}, \widetilde{\boldsymbol{\omega}})$ of the vehicle's ego-velocity. {In particular, dynamic objects have an extra velocity component that is inconsistent with \eqref{eq: doppler relation}, so they appear as outliers and we can explicitly remove them by RANSAC. Since most points in a 4D radar scan are static in practice, this step reliably filters dynamic points and ensures that the point clouds used for registration consist mainly of static points. Nevertheless, in scenes with an unusually high density of dynamic objects, the static-point majority assumption required by RANSAC may be violated and the performance could degrade. We identify this as a limitation of the current framework.}

Now that given the estimated ego-velocity $(\widetilde{\mathbf{v}}, \widetilde{\boldsymbol{\omega}})$, the relative motion between two frames of radar point clouds over a time interval $\Delta t$ is approximated as
\begin{equation}
\widehat{\mathbf{T}}
= \exp\!\left(
\begin{bmatrix}
\widetilde{\boldsymbol{\omega}}\, \Delta t \\
\widetilde{\mathbf{v}}\, \Delta t
\end{bmatrix}
\right), \label{eq: ego-motion}
\end{equation}
which also provides an initial estimate of the transformation. Using such transformation, the two point clouds are mapped into a unified coordinate frame. We then identify the spatial region that is jointly observable in both scans and retain only the points falling within this region.
This ensures that the source and target point clouds are approximately sampled from the same distribution. This pre-processing step guarantees the validity of the theoretical analysis in Theorem \ref{thm: main_thm} in the practical implementation and improves the robustness and accuracy of our proposed method.

\begin{figure}
  \centering
  \includegraphics[width=0.5\textwidth]{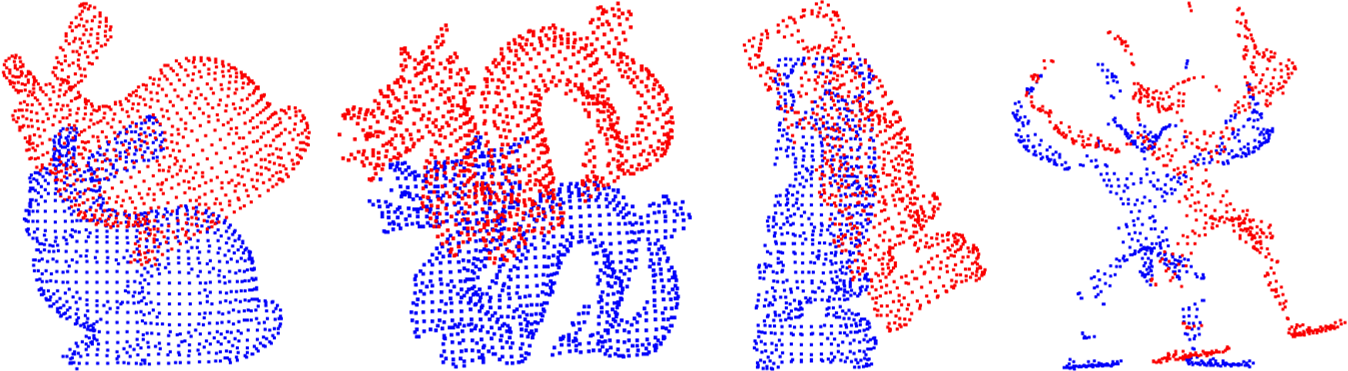}
  \caption{The source (red) and target (blue) point clouds for synthetic registration. From left to right: Bunny, Dragon, Buddha, Armadillo.}
  \label{fig: bunny_and_dragon}
\end{figure}

\begin{table*}
\centering
\caption{Translation and rotation error (m / deg) on noiseless point clouds.}
\label{tab: noiseless_results}
\begin{tabular*}{\textwidth}{@{\extracolsep{\fill}} lcccc}
\toprule
\textbf{Method} & \textbf{Bunny} & \textbf{Dragon} & \textbf{Buddha} & \textbf{Armadillo} \\
\midrule
GICP \cite{segal2009generalized} &
$3.38 \times 10^{-3} \, / \, 2.39 \times 10^{-2}$ &
$6.91 \times 10^{-2} \, / \, 0.49$ &
$9.12 \times 10^{-3} \, / \, 9.83 \times 10^{-2}$ &
$7.43 \times 10^{-4} \, / \, 3.88 \times 10^{-3}$ \\
NDT \cite{magnusson2009three} &
$4.75 \times 10^{-3} \, / \, 2.74 \times 10^{-2}$ &
$1.03 \times 10^{-2} \, / \, 6.89 \times 10^{-2}$ &
$4.17 \times 10^{-3} \, / \, 2.15 \times 10^{-2}$ &
$3.79 \times 10^{-3} \, / \, 2.04 \times 10^{-2}$ \\
{CPD} \cite{myronenko2010point} &
$6.95 \times 10^{-3} \, / \, 5.23 \times 10^{-2}$ &
$2.68 \times 10^{-3} \, / \, 1.32 \times 10^{-2}$ &
$2.72 \times 10^{-3} \, / \, 1.97 \times 10^{-2}$ &
$2.40 \times 10^{-3} \, / \, 9.94 \times 10^{-3}$ \\
{GMMReg} \cite{jian2010robust} &
$9.73 \times 10^{-8} \, / \, 0$ &
$1.49 \times 10^{-8} \, / \, 0$ &
$6.57 \times 10^{-8} \, / \, 0$ &
$5.17 \times 10^{-8} \, / \, 0$ \\
TEASER++ \cite{yang2020teaser} &
$4.66 \times 10^{-8} \, / \, 0$ &
$4.17 \times 10^{-8} \, / \, 0$ &
$8.38 \times 10^{-8} \, / \, 0$ &
$8.23 \times 10^{-8} \, / \, 0$ \\
PNLK-R \cite{li2021pointnetlk} &
$\mathbf{1.68 \times 10^{-8}} \, / \, 0$ &
$3.91 \times 10^{-8} \, / \, 4.88 \times 10^{-4}$ &
$1.00 \times 10^{-8} \, / \, 0$ &
$\mathbf{3.68 \times 10^{-8}} \, / \, 0$ \\
\textbf{Ours} &
$2.23 \times 10^{-8} \, / \, \mathbf{0}$ &
$\mathbf{1.19 \times 10^{-8}} \, / \, \mathbf{0}$ &
$\mathbf{1.85 \times 10^{-8}} \, / \, \mathbf{0}$ &
$4.47 \times 10^{-8} \, / \, \mathbf{0}$ \\
\bottomrule
\end{tabular*}

\vspace{2mm}
{\footnotesize
\textit{Note:} (1) ``PNLK-R'' denotes the PointNetLK Revisited method \cite{li2021pointnetlk}. (2) Errors smaller than $10^{-12}$ are reported as zero.
}
\end{table*}

\section{Experiments}\label{sec: experiments}

\subsection{Experimental Setup}
In our experiments, we evaluate the performance of our proposed GMMR method on both synthetic and real-world datasets. In particular, the experiments include two settings: pairwise registration on synthetic data and odometry and SLAM evaluation on real-world 4D radar datasets.

For the synthetic pairwise registration experiments, we use the \textbf{Stanford Bunny}, \textbf{Dragon}, \textbf{Happy Buddha} and \textbf{Armadillo} datasets obtained from the Stanford 3D Scanning Repository \cite{curless1996volumetric}. In particular, to emulate the characteristics of 4D radar point clouds, we down-sample each point cloud and add synthetic noise. Then we apply known rigid transformations to generate a second frame for each point cloud. Notably, in this case, the generated noisy source and target point clouds are not perfectly aligned under the ``true" transformation.

Moreover, for real-world evaluation experiments, we use the real-world \textbf{4D radar dataset} \cite{li20234d} in our previous work. In particular, the dataset consists of outdoor data sequences measured by a 4D radar sensor. And on this dataset, we first conduct scan-to-scan odometry evaluation. Moreover, we also integrate our GMMR into a 4D radar-only SLAM pipeline \cite{li20234d} to perform registration in a full SLAM system and evaluate the performance in realistic scenarios. {We further evaluate the generalization ability of our method on the public 4D radar dataset \textbf{NTU4DRadLM} \cite{zhang2023ntu4dradlm}, by integrating GMMR into another SLAM system 4DRadarSLAM \cite{zhang20234dradarslam}.}

In addition, all experiments are run on an Intel i7-1360P computer with 64GB RAM and an NVIDIA RTX3080Ti GPU.

\subsection{Synthetic experiments for Pairwise Registration} \label{sec: synthetic experiments}
We first evaluate the performance of GMMR on synthetic point clouds. In particular, we use ``bun000.ply'' from the Bunny dataset, ``dragonStandRight\_0.ply'' from the Dragon dataset, ``happyStandRight\_0.ply'' from the Buddha dataset, and ``ArmadilloBack\_0.ply'' from the Armadillo dataset. {To emulate the sparsity characteristics of 4D radar measurements, each model is down-sampled using the voxel grid filtering algorithm provided by the Point Cloud Library (PCL) \cite{rusu20113d}, resulting in point clouds containing 980, 1078, 791, and 641 points, respectively. The same point clouds are used for all compared methods below to ensure a fair comparison.} And a rigid transformation is then manually applied to generate the source and target frames, as illustrated in Fig. \ref{fig: bunny_and_dragon}. The goal is to recover the applied transformation as accurately as possible. Consequently, we evaluate the registration performance using two error metrics: the translation error (m) and the rotation error (deg). The errors are computed between the estimated transformation $\hat{\theta}_{M,N}$ and the ground truth transformation $\theta_0$ as follows
\begin{align}
\mbox{translation error} &:= \| \Trans( \mat(\theta_0)^{-1} \mat(\hat{\theta}_{M,N})) \|, \label{eq: tl error}\\
\text{rotation error} &:= \Rot( \mat(\theta_0)^{-1} \mat(\hat{\theta}_{M,N})), \label{eq: rot error}
\end{align}
where $\mat(\theta) := \begin{bmatrix}\mathbf{R} & \mathbf{t} \\ \mathbf{0}^\top & 1 \end{bmatrix}$, 
$\Trans(\mat(\theta)) = \mathbf{t}$, $\Rot(\mat(\theta)) = \arccos \frac{\trace(\mathbf{R})-1}{2}$ and $\trace(\cdot)$ denotes the trace of a matrix.

We compare our registration method with several representative baselines, including NDT and GICP from PCL \cite{rusu20113d}, CPD \cite{myronenko2010point} and GMMReg \cite{jian2010robust} from \cite{probreg}, the robust registration method TEASER++ \cite{yang2020teaser}, and the learning-based method PointNetLK Revisited \cite{li2021pointnetlk}. More specifically, for PointNetLK Revisited, we use the pretrained model released by the authors, which is trained on the ModelNet40 dataset. {In all synthetic experiments, the optimization is initialized with the identity transformation.}

\begin{figure}
  \centering
  \includegraphics[width=0.5\textwidth]{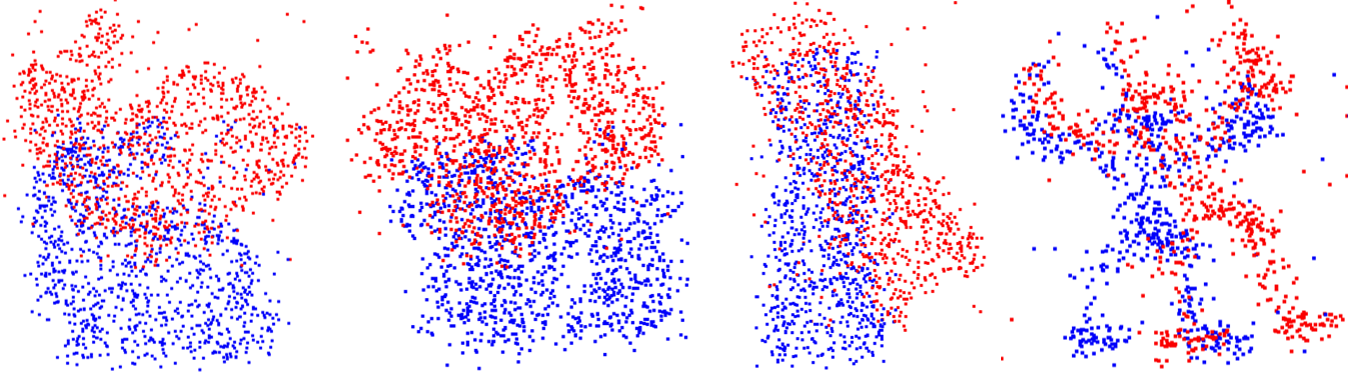}
  \caption{The noisy source (red) and target (blue) point clouds for synthetic registration. From left to right: Noisy Bunny, Noisy Dragon, Noisy Buddha, Noisy Armadillo.}
  \label{fig: noisy_bunny_and_dragon}
\end{figure}

\begin{table*}
\centering
\caption{Translation and rotation error (m / deg) on noisy point clouds.}
\label{tab: noisy_results}
\begin{tabular*}{\textwidth}{@{\extracolsep{\fill}} lcccc}
\toprule
\textbf{Method} & \textbf{Noisy Bunny} & \textbf{Noisy Dragon} & \textbf{Noisy Buddha} & \textbf{Noisy Armadillo} \\
\midrule
GICP \cite{segal2009generalized} &
$1.59 \times 10^{-2} \, / \, 0.14$ &
$0.11 \, / \, 0.53$ &
$2.92 \times 10^{-2} \, / \, 0.15$ &
$4.48 \times 10^{-2} \, / \, 0.24$ \\
NDT \cite{magnusson2009three} &
$5.63 \times 10^{-3} \, / \, 4.38 \times 10^{-2}$ &
$3.89 \times 10^{-3} \, / \, 1.54 \times 10^{-2}$ &
$9.24 \times 10^{-3} \, / \, 6.38 \times 10^{-2}$ &
$3.33 \times 10^{-3} \, / \, 1.71 \times 10^{-2}$ \\
{CPD} \cite{myronenko2010point} &
$5.25 \times 10^{-3} \, / \, 4.97 \times 10^{-2}$ &
$1.22 \times 10^{-2} \, / \, 5.28 \times 10^{-2}$ &
$7.63 \times 10^{-3} \, / \, 3.61 \times 10^{-2}$ &
$2.80 \times 10^{-2} \, / \, 0.177$ \\
{GMMReg} \cite{jian2010robust} &
$7.63 \times 10^{-3} \, / \, 0.112$ &
$1.42 \times 10^{-2} \, / \, 9.03 \times 10^{-2}$ &
$7.61 \times 10^{-3} \, / \, 5.48 \times 10^{-2}$ &
$1.14 \times 10^{-2} \, / \, 5.41 \times 10^{-2}$ \\
TEASER++ \cite{yang2020teaser} &
$3.86 \times 10^{-3} \, / \, 2.90 \times 10^{-2}$ &
$8.20 \times 10^{-3} \, / \, 3.90 \times 10^{-2}$ &
$5.11 \times 10^{-3} \, / \, 3.09 \times 10^{-2}$ &
$5.98 \times 10^{-3} \, / \, 3.18 \times 10^{-2}$ \\
PNLK-R \cite{li2021pointnetlk} &
$1.01 \times 10^{-2} \, / \, 4.48 \times 10^{-2}$ &
$1.23 \times 10^{-1} \, / \, 5.77 \times 10^{-1}$ &
$9.20 \times 10^{-2} \, / \, 4.84 \times 10^{-1}$ &
$7.32 \times 10^{-2} \, / \, 4.27 \times 10^{-1}$ \\
\textbf{Ours} &
$\mathbf{1.21 \times 10^{-3}} \, / \, \mathbf{2.28 \times 10^{-2}}$ &
$\mathbf{3.06 \times 10^{-3}} \, / \, \mathbf{9.98 \times 10^{-3}}$ &
$\mathbf{1.89 \times 10^{-3}} \, / \, \mathbf{1.72 \times 10^{-2}}$ &
$\mathbf{1.68 \times 10^{-3}} \, / \, \mathbf{1.21 \times 10^{-2}}$ \\
\bottomrule
\end{tabular*}
\end{table*}

\subsubsection{Noiseless registration}
We first evaluate our method on noiseless point clouds. The quantitative results on noiseless point clouds are summarized in Table \ref{tab: noiseless_results}.

\subsubsection{Noisy registration}
To further evaluate the robustness of our method, we conduct additional registration experiments under noisy point clouds. In this setting, random noise is drawn from the distributions to come independently twice and added to the source and the target point clouds, respectively. This makes sure that the two frames do not align perfectly under the ``true" transformation. Such simulation reflects the fact that the 4D radar measurement noise typically exhibit frame-to-frame variability. In particular, two types of artificial noise are introduced to emulate the characteristics of 4D radar:
\begin{itemize}
    \item Zero mean Gaussian additive noise: let $\mathbf{p}^s_i$ and $\mathbf{p}^t_i$ be the pre-generated source and target point clouds, respectively. That is, $\mathbf{p}^t_i=T(\mathbf{p}^s_i;\theta_0)$.
    To further generate the point clouds that are used for registration, we first add zero mean Gaussian additive noise to both source and target frames to simulate the point measurement noise, namely,
    \begin{align*}
        \mathbf{x}_i=\mathbf{p}^s_i+\mathbf{w}_i,\quad \mathbf{y}_j=\mathbf{p}^t_j+\mathbf{v}_j,
    \end{align*}
    where $\mathbf{w}_i\sim\mathcal{N}(0,0.005I)$, $\mathbf{y}_i\sim\mathcal{N}(0,0.005I)$ for all $i=1,\ldots,M$ and $j=1,\ldots,N$.
    \item Uniform outlier points: We further inject noisy points that are independently sampled from a uniform distribution whose support covers the spatial extents of the source and the target clouds. It imitates the environmental clutters. The number of outliers is 10\% of the original cloud size.
\end{itemize}
Examples of the generated noisy point clouds are shown in Fig. \ref{fig: noisy_bunny_and_dragon}. Translation and rotation errors are computed using \eqref{eq: tl error} and \eqref{eq: rot error}. Table \ref{tab: noisy_results} presents the results. It demonstrates that our GMMR maintains significantly lower errors than the baselines even in the presence of substantial noise.

Moreover, we also compare the performance of our method and baselines under varying noise levels. Fig. \ref{fig: tl_gauss}, \ref{fig: rot_gauss}, \ref{fig: tl_outlier}, \ref{fig: rot_outlier} illustrate the translation and rotation errors on the Bunny point cloud for increasing standard deviations of Gaussian noise and outlier ratios, respectively. As shown, our method consistently outperforms other baseline methods, including at extreme noise levels with up to 50\% outliers. Notably, the performance of our method only partially degrades as the noise level increases, hence it further demonstrates its robustness to noise.

\subsubsection{Partial-overlap robustness}
In addition to noise robustness, we further evaluate the performance of our method under varying levels of partial overlap. In particular, we remove a percentage of points from the source point cloud by random sampling to simulate different overlap ratios. Fig. \ref{fig: tl_overlap} and \ref{fig: rot_overlap} show the translation and rotation errors on the Bunny dataset under decreasing overlap. Consequently, TEASER++ achieves the best accuracy across different overlap ratios due to its correspondence-based formulation and simultaneous pose and correspondences process. Although our method does not achieve the accuracy of TEASER++, the performance gap still remains small. Moreover, across a wide range of overlap ratios, our method consistently outperforms other baseline methods, including extremely low overlap down to 50 \%.

Overall, these experiments demonstrate that, in spite of substantial sparsity, heavy noise of the point clouds and partial overlap, our method achieves more accurate and robust registration performance than other existing registration baselines.

\begin{figure*}
    \centering

    \begin{subfigure}{0.32\textwidth}
        \centering
        \includegraphics[width=\textwidth]{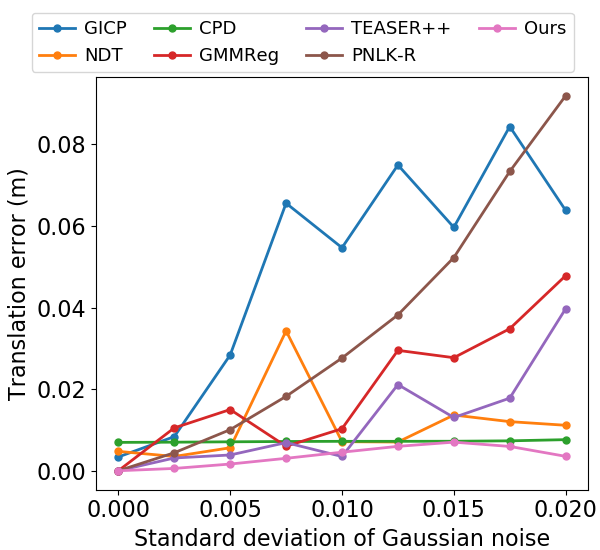}
        \caption{Translation error vs. $\sigma$}
        \label{fig: tl_gauss}
    \end{subfigure}
    \hfill
    \begin{subfigure}{0.32\textwidth}
        \centering
        \includegraphics[width=\textwidth]{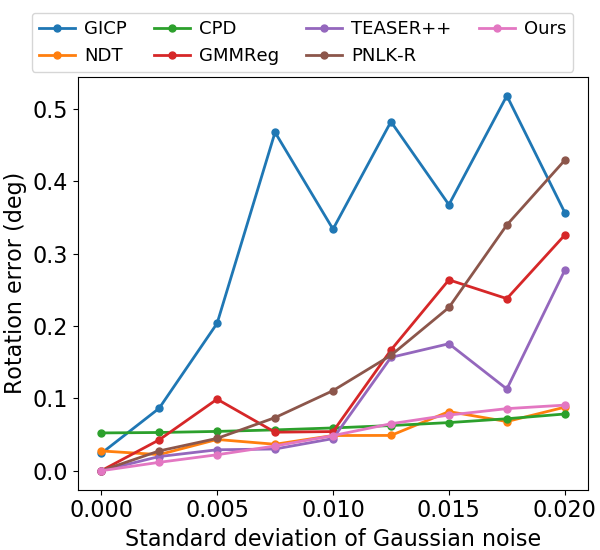}
        \caption{Rotation error vs. $\sigma$}
        \label{fig: rot_gauss}
    \end{subfigure}
    \hfill
    \begin{subfigure}{0.32\textwidth}
        \centering
        \includegraphics[width=\textwidth]{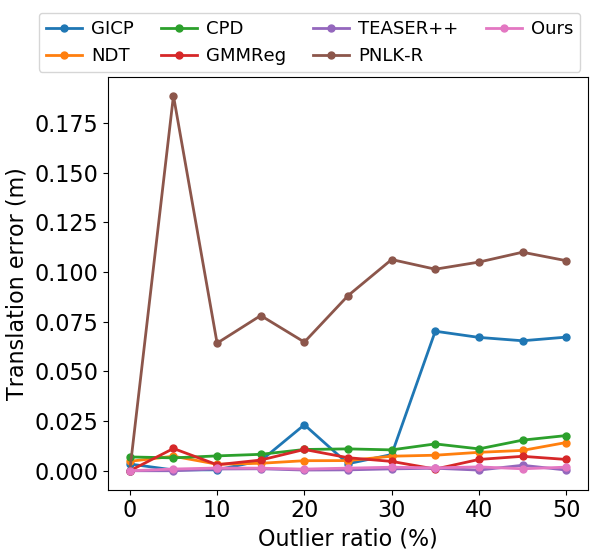}
        \caption{Translation error vs. outlier ratio}
        \label{fig: tl_outlier}
    \end{subfigure}

    \par\medskip

    \begin{subfigure}{0.32\textwidth}
        \centering
        \includegraphics[width=\textwidth]{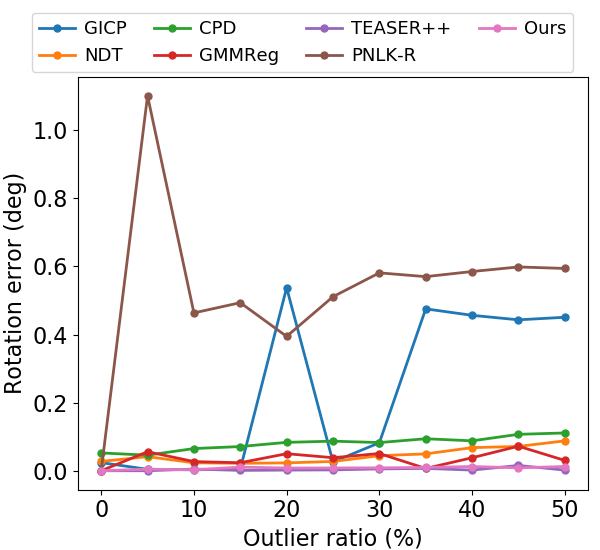}
        \caption{Rotation error vs. outlier ratio}
        \label{fig: rot_outlier}
    \end{subfigure}
    \hfill
    \begin{subfigure}{0.32\textwidth}
        \centering
        \includegraphics[width=\textwidth]{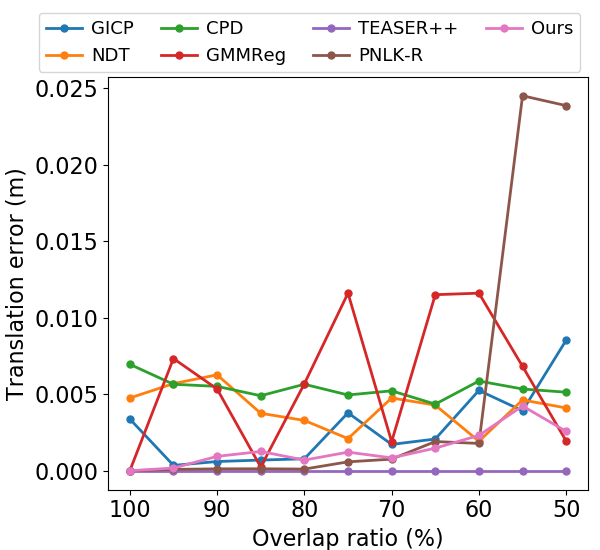}
        \caption{Translation error vs. overlap ratio}
        \label{fig: tl_overlap}
    \end{subfigure}
    \hfill
    \begin{subfigure}{0.32\textwidth}
        \centering
        \includegraphics[width=\textwidth]{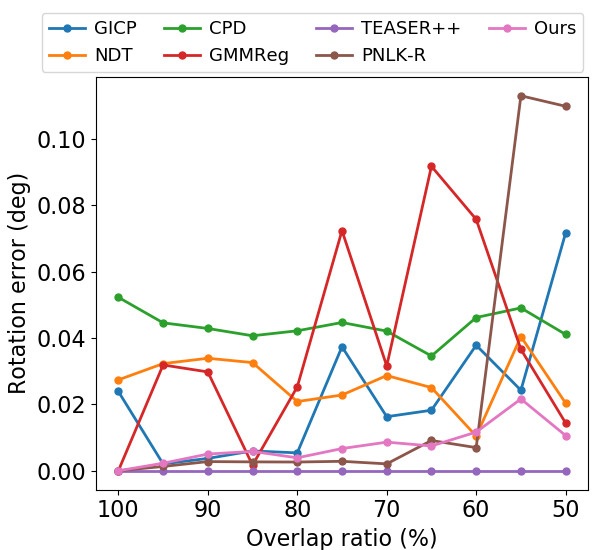}
        \caption{Rotation error vs. overlap ratio}
        \label{fig: rot_overlap}
    \end{subfigure}

    \caption{{Comparison of translation and rotation errors under different Gaussian noise, outlier ratios, and overlap ratios.}}
    \label{fig: error_comparison}
\end{figure*}

\begin{figure*}
    \centering

    \begin{subfigure}{0.32\textwidth}
        \centering
        \includegraphics[width=\textwidth]{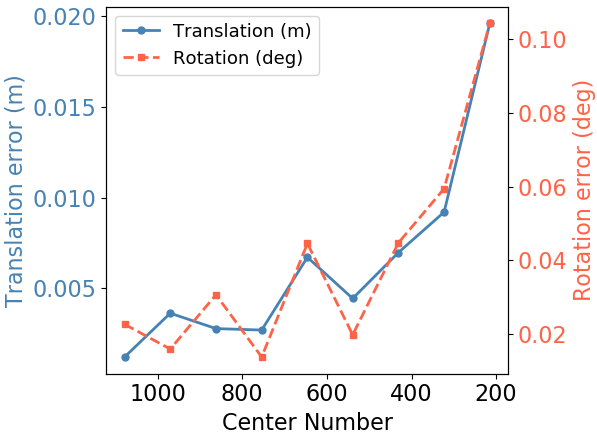}
        \caption{Registration error vs. $\kappa$}
        \label{fig: error_center}
    \end{subfigure}
    \hfill
    \begin{subfigure}{0.32\textwidth}
        \centering
        \includegraphics[width=\textwidth]{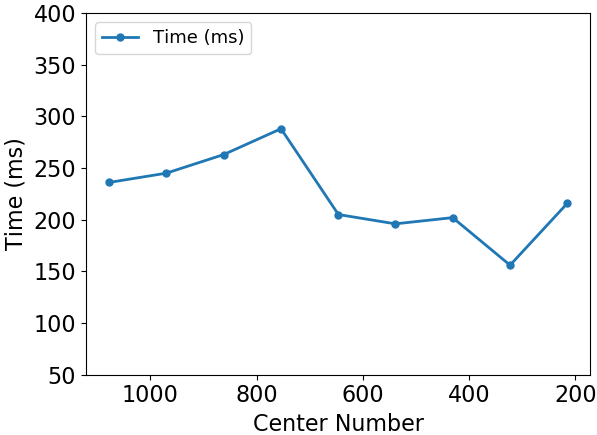}
        \caption{Runtime vs. $\kappa$}
        \label{fig: time_center}
    \end{subfigure}
    \hfill
    \begin{subfigure}{0.32\textwidth}
        \centering
        \includegraphics[width=\textwidth]{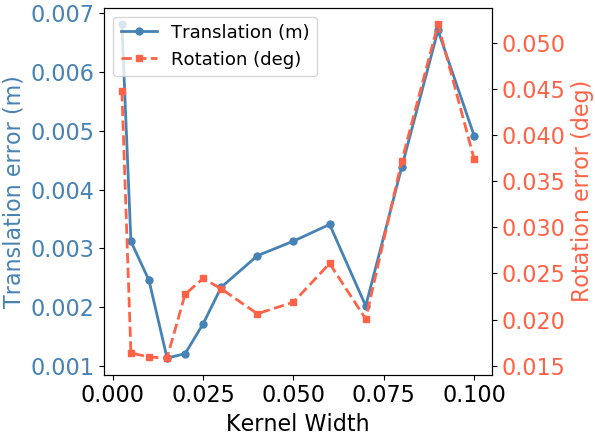}
        \caption{Registration error vs. $\Sigma$}
        \label{fig: error_width}
    \end{subfigure}

    \caption{Sensitivity analysis with respect to the number of centers $\kappa$ and the kernel width $\Sigma$.}
    \label{fig: param_sensitivity}
\end{figure*}

{\subsubsection{Hyperparameter Sensitivity}}
{We further conduct a sensitivity analysis of the number of centers $\kappa$ and the kernel width $\Sigma$ on the noisy Bunny point cloud, as shown in Fig. \ref{fig: param_sensitivity}.}

{For the number of centers $\kappa$, the registration error remains relatively stable over a broad range and increases only when $\kappa$ becomes too small to capture the global structure of the point distribution. Meanwhile, the runtime does not vary significantly with $\kappa$, since the moment computation is parallelized using CUDA: when $\kappa$ remains within the effective range supported by the GPU cores, the kernel evaluations run simultaneously, and the overall runtime is mainly dominated by the number of optimization iterations rather than the number of centers itself.}

{For the kernel width $\Sigma$, the registration error exhibits a U-shaped trend. A too small $\Sigma$ makes the Gaussian RBF kernels overly localized and provides insufficient gradient information for sparse point clouds, while a too large $\Sigma$ reduces the ability to distinguish local features. Consequently, the best registration performance is achieved within an intermediate range of $\Sigma$.}

{Overall, the results demonstrate that our method is robust to hyperparameter selection within a reasonably wide parameter range, and the parameters in the experiments are selected accordingly.}

\begin{table*}
\centering
\caption{Comparison of the relative and absolute errors (\% / deg/m / m) in different real-world scenes.}
\label{tab: quantitativeresults}
\begin{tabular*}{\textwidth}{@{\extracolsep{\fill}} lccc}
\toprule
\textbf{SLAM Method} & \textit{Campus 1} & \textit{Campus 2} & \textit{Campus 3} \\
\midrule
LiDAR SLAM \cite{kim_scan_2018, shan2018lego} &
1.96 / 0.0209 / 2.27 &
3.19 / 0.0120 / 2.57 &
2.68 / 0.0173 / 3.33 \\
\midrule
4DRaSLAM \cite{li20234d} &
2.32 / 0.0216 / 2.28 &
3.13 / 0.0201 / \textbf{3.79} &
3.06 / 0.0245 / 3.83 \\
\textbf{Ours} &
\textbf{2.01 / 0.0148 / 2.06} &
\textbf{2.64 / 0.0132} / 6.32 &
\textbf{2.62 / 0.0153 / 3.35} \\
\bottomrule
\end{tabular*}
\end{table*}

\begin{figure*}
  \centering
  \includegraphics[width=\textwidth]{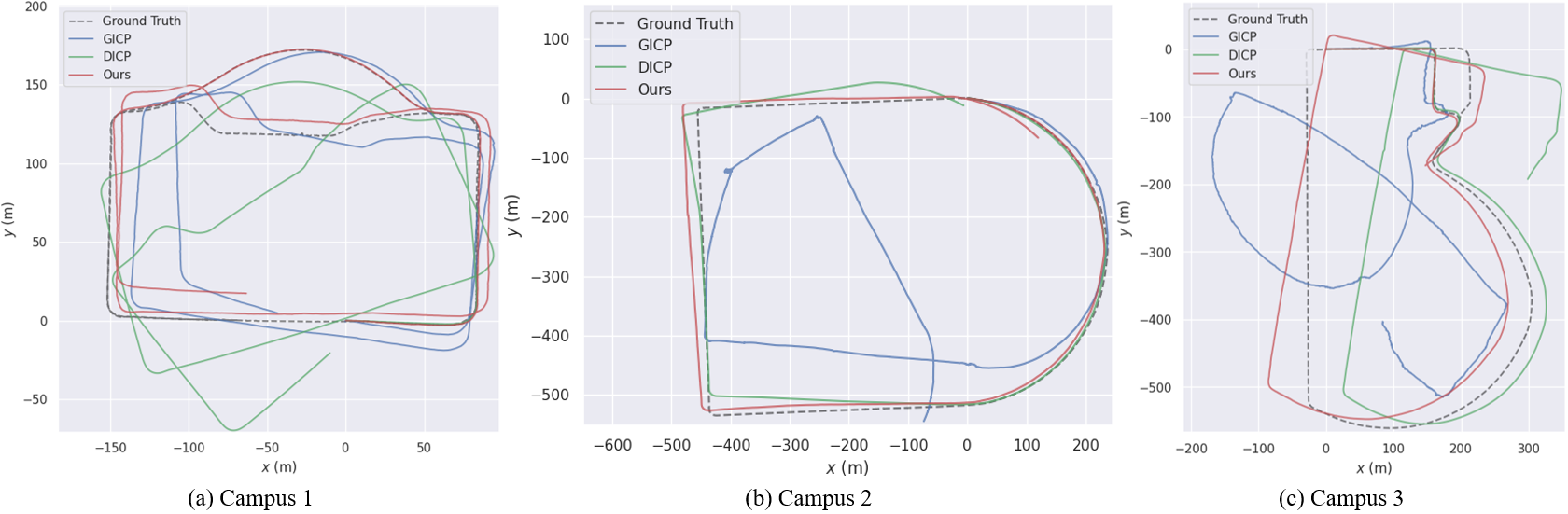}
  \caption{{Comparison of trajectories on \textit{Campus 1}, \textit{Campus 2} and \textit{Campus 3} in 4D radar dataset.}
  \label{fig: traj}}
\end{figure*}

\begin{figure*}
  \centering
  \includegraphics[width=\textwidth]{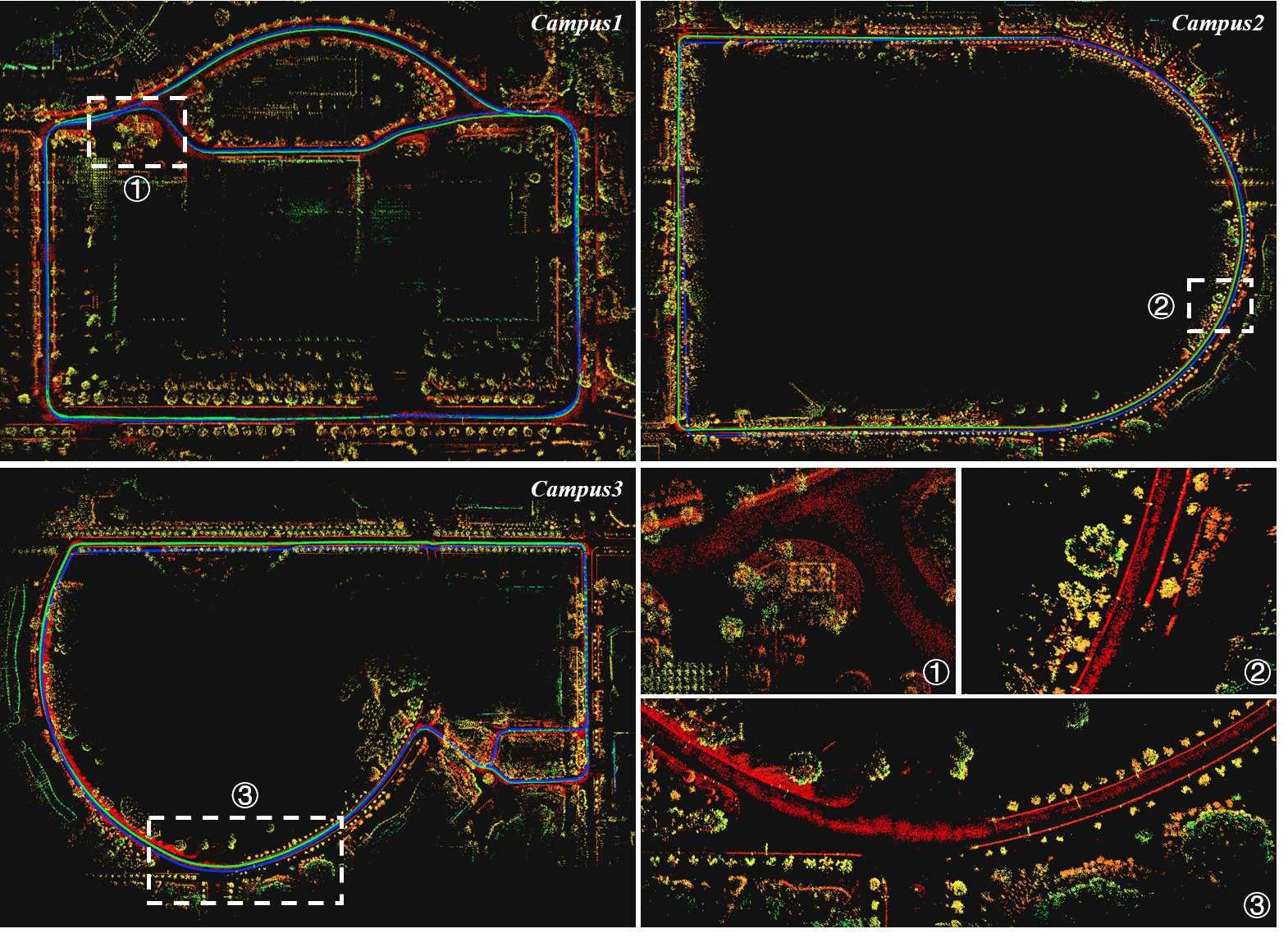}
  \caption{Reconstructed 4D radar point cloud maps of \textit{Campus 1}, \textit{Campus 2}, and \textit{Campus 3} from the 4D radar dataset using the proposed GMMR-based SLAM system. The estimated SLAM trajectories are shown in green, while the ground-truth trajectories are shown in blue.}
  \label{fig: map}
\end{figure*}

\subsection{Real-world Experiments}
\subsubsection{Scan-to-scan odometry evaluation}
To demonstrate the practical utility of our method on 4D radar data, we first test GMMR in a scan-to-scan odometry pipeline. In particular, we implement a scan-to-scan radar odometry using our method and other baselines used in Section \ref{sec: synthetic experiments}. {We also include DICP \cite{hexsel2022dicp} as a baseline, which leverages Doppler velocity measurements for registration and is evaluated on the real-world dataset only, since Doppler measurements are unavailable in the synthetic setting. The initial transformation for all methods is provided by the Doppler-based ego-motion estimation from \eqref{eq: ego-motion}.} Then we compare their performance on the real-world 4D radar dataset collected in our previous work \cite{li20234d}.

{From the experiments, we observe that NDT, CPD, GMMReg, TEASER++, and PointNetLK Revisited exhibit substantial performance degradation on real-world 4D radar scans. We conjecture that this is because NDT, CPD and GMMReg rely on stable local Gaussian statistics, soft probabilistic correspondences through EM optimization, and accurate density overlap estimation, respectively, which become unreliable under sparse and noisy radar data. TEASER++, on the other hand, depends on reliable feature correspondences, which are difficult to establish due to the lack of distinct local structures in radar point clouds. Moreover, PointNetLK Revisited is pretrained on dense and clean ModelNet40 shapes and may not generalize well to sparse and noisy radar data.} In contrast, our method yields more stable and accurate odometry, as shown in Fig. \ref{fig: traj}. For clarity, the highly inconsistent trajectories produced by the failed methods are not included in the figure. Nevertheless, our current implementation runs at a speed of approximately 158 ms per frame. It indicates that the algorithm's computational efficiency still needs further improvement, which we will consider in future work.

\subsubsection{SLAM evaluation}
\paragraph{Performance on the collected dataset:} Next, we integrate our registration method into our previous work 4DRaSLAM \cite{li20234d} to evaluate the performance in a complete SLAM system. In particular, we first pre-process the 4D radar point clouds as in \cite{li20234d} and estimate the ego-velocity. We then incorporate our scan-to-submap GMMR with overlap extraction as the scan registration module. Unlike previous pairwise registration, this experiment registers consecutive frames to submaps to build a globally consistent map. Since scan-to-submap registration is triggered only at keyframes and typically runs at a lower frequency than the raw scan rate, the overall SLAM system still remains real-time. To evaluate the performance, we compute the Relative Error (RE) in the same way as that of the KITTI odometry benchmark \cite{geiger_are_2012}. In particular, we compute the mean translation and rotation errors from length 100 m to 800 m with a 100 m increment (For more details, please refer to \cite{geiger_are_2012}). In addition, translation errors are in percentage (\%), and rotation errors are in degrees per meter (deg/m). Besides RE, we also compute the absolute error using the Root-Mean-Square Error (RMSE) of absolute trajectory error (ATE) in meters (m) to evaluate the global consistency and overall accuracy of the estimated trajectory \cite{sturm2012benchmark}.

The results are listed in Table \ref{tab: quantitativeresults} in the format of ``translation error (\%) / rotation error (deg/m) / ATE (m)". More precisely, the results show that our method effectively handles the sparsity and noise issues of 4D radar points and achieves accurate and robust pose estimations that are comparable to those of LiDAR-based SLAM \cite{kim_scan_2018, shan2018lego}. 

In addition to the trajectory accuracy, we further evaluate the mapping quality produced by our GMMR-based SLAM system. Fig. \ref{fig: map} illustrates the reconstructed 4D radar point cloud map, along with enlarged views of representative regions. The map preserves fine structural details such as trees, small vegetation island and planting strips, while maintaining strong global consistency. It demonstrates the effectiveness and accuracy of the proposed registration method despite the sparsity and noise of 4D radar measurements.

{\paragraph{Ablation study:} We further conduct an ablation experiment on the ego-motion-based overlap extraction module. The results are shown in Table \ref{tab: ablation}. Removing the overlap extraction module leads to only a slight performance degradation, suggesting that the core moment-matching objective contributes significantly to the overall performance of our method. The overlap extraction serves as a lightweight preprocessing step that helps ensure the shared-distribution assumption is satisfied in practice.}

\begin{table*}
\centering
\caption{Ablation study on the ego-motion-based overlap extraction module.}
\label{tab: ablation}
\begin{tabular*}{\textwidth}{@{\extracolsep{\fill}} lccc}
\toprule
\textbf{Method} & \textit{Campus 1} & \textit{Campus 2} & \textit{Campus 3} \\
\midrule
Ours w/o overlap extraction &
2.12 / 0.0164 / 2.13 &
2.89 / 0.0145 / 6.18 &
2.73 / 0.0187 / 3.54 \\
Ours &
2.01 / 0.0148 / 2.06 &
2.64 / 0.0132 / 6.32 &
2.62 / 0.0153 / 3.35 \\
\bottomrule
\end{tabular*}
\end{table*}

{\paragraph{Evaluation on NTU4DRadLM dataset:} To further evaluate the generalization ability of our method, we validate it on the public 4D radar dataset NTU4DRadLM \cite{zhang2023ntu4dradlm}. More specifically, we integrate our registration method into 4DRadarSLAM \cite{zhang20234dradarslam}, which was originally evaluated on this dataset. It is a full radar SLAM system that employs a scan-to-scan registration module based on Adaptive Probability Distribution GICP (APDGICP), considering the probability distribution of each point. We replace its APDGICP-based registration with our GMMR and compare the results against those reported in the original paper. For a fair comparison, the performance is evaluated using the same trajectory evaluation tool \cite{zhang2018tutorial} as used in \cite{zhang20234dradarslam}, reporting both RE and ATE. The results are listed in Table \ref{tab: ntu_results}. It demonstrates that our method maintains competitive performance on the new dataset, providing stronger evidence of its generalization ability.}

\begin{table*}
\centering
\caption{Comparison of the relative and absolute errors (\% / deg/m / m) on NTU4DRadLM dataset.}
\label{tab: ntu_results}\
\resizebox{\textwidth}{!}{
\begin{tabular}{lccccc}
\toprule
\textbf{Method} & \textit{cp} & \textit{garden} & \textit{nyl} & \textit{loop 2} & \textit{loop 3}\\
\midrule
APDGICP \cite{zhang20234dradarslam} &
3.02 / 0.0448 / 2.35 &
2.05 / 0.0309 / \textbf{2.41} &
2.85 / 0.0131 / 11.37 &
5.79 / 0.0100 / 84.88 &
\textbf{4.03} / 0.0069 / 43.67 \\
\textbf{Ours} &
\textbf{1.49 / 0.0327 / 1.14} &
\textbf{1.77 / 0.0208} / 3.89 &
\textbf{2.23 / 0.0127 / 7.54} &
\textbf{4.66 / 0.0052 / 66.19} &
4.62 / \textbf{0.0056 / 29.49} \\
\bottomrule
\end{tabular}
}
\end{table*}

\section{Conclusions}

In this work, we introduce a 4D radar point cloud registration method based on Generalized Method of Moments. In particular, our approach uses Gaussian radial basis function kernels to compute the generalized moments of the point cloud. Then the transformation is estimated through minimizing the moment matching loss. We also show the consistency of our method. The experiments demonstrate that our method achieves better accuracy and stability compared to other registration approaches. We further integrate this framework into a 4D radar-based SLAM system. The results show that our method improves the localization performance to levels that is similar to LiDAR-based SLAM.

In future work, we aim to further enhance the real-time performance of the algorithm. Moreover, we will extend the framework to support non-rigid transformations. These improvements will broaden the applicability of our method to more dynamic and complex environments.








\bibliographystyle{elsarticle-num}

\bibliography{reference}



\end{document}